\documentclass[accepted]{uai2023arxiv} 

\usepackage[american]{babel}

\usepackage{natbib} 
\usepackage{todonotes}
    \setuptodonotes{inline}
\usepackage{amssymb}
\usepackage{amsthm}
\usepackage{xcolor}
\usepackage{mathtools} 
\usepackage{booktabs} 
\usepackage{tikz} 
    \usetikzlibrary{bayesnet}
    \tikzstyle{obs} = [latent,fill=gray!40]

\usepackage{subcaption}

\newtheorem{theorem}{Theorem}[section]

\renewcommand{\paragraph}[1]{\textbf{#1}~}

\newcommand{\OurTitle}{Bayesian Quantification with Black-Box Estimators}
\title{\OurTitle{}}

\newcommand{\AppendixFastInference}{A}
\newcommand{\AppendixAsymptoticIdentifiability}{B}
\newcommand{\AppendixExperiments}{C}
\newcommand{\AppendixDiscretization}{D}
\newcommand{\AppendixQuantificationMethods}{E}

\author[1]{Albert Ziegler}
\author[2]{Pawe{\l} Czy{\.z}}
\affil[1]{%
    GitHub Next\\
    GitHub, Inc.\\
    San Francisco\\
    USA
}
\affil[2]{%
    ETH AI Center and Department of Biosystems Science and Engineering\\
    ETH Z{\"u}rich\\
    Switzerland
}

\begin{document}
\maketitle

\begin{abstract}
  Understanding how different classes are distributed in an unlabeled data set is an important challenge for the calibration of probabilistic classifiers and uncertainty quantification.
  Approaches like adjusted classify and count, black-box shift estimators, and invariant ratio estimators use an auxiliary (and potentially biased) black-box classifier trained on a different (shifted) data set to estimate the class distribution and yield asymptotic guarantees under weak assumptions.
  We demonstrate that all these algorithms are closely related to the inference in a particular Bayesian model, approximating the assumed ground-truth generative process.
  Then, we discuss an efficient Markov Chain Monte Carlo sampling scheme for the introduced model and show an asymptotic consistency guarantee in the large-data limit.
  We compare the introduced model against the established point estimators in a variety of scenarios, and show it is competitive, and in some cases superior, with the state of the art.
\end{abstract}

\section{Introduction}
Consider a medical test predicting illness (classification label $Y$), such as influenza, based on symptoms (features $X$), such as high fever.
This often can be modelled as an anti-causal problem\footnote{While influenza causes high fever, in many medical problems the causal relationships are much more complex \citep{castro2020}.} \citep{Scholkopf2012}, where $Y$ causally affects $X$. 
Under the usual i.i.d assumption, one can approximate the probabilities $P(Y\mid X)$ using the training data set.

However, the performance on real-world data may be lower than expected, due to data shift: the issue that real-world data comes from a different probability distribution than training data. 
For example, classifiers trained during early stages of the COVID-19 pandemic will underestimate the incidence of the illness at the time of surge in infections. 

The paradigmatic case of data shift is \textit{prior probability shift}, where the context 
influences the distribution of the target label $Y$, although the generative
mechanism generating $X$ from $Y$ is left unchanged. In  other words,
\newcommand{\Ptest}{P_\text{test}}
\newcommand{\Ptrain}{P_\text{train}}
\(
    \Ptrain(X\mid Y) = \Ptest(X\mid Y),
\)
although $\Ptrain(Y)$ may differ from $\Ptest(Y)$.
If $\Ptest(Y)$ is known, then $\Ptest(Y\mid X)$ can be calculated by rescaling $\Ptrain(Y\mid X)$ according to Bayes' theorem (see \citet[\textsection 2.2]{Saerens-2001-adjustingtheoutputs} or \citet[\textsection 3.2]{Scholkopf2012}), 
and is conceptually similar to importance weighting in training classifiers on unbalanced data set \citep[\textsection 3.2]{Kouw2019}.

However, $\Ptest(Y)$ is usually unknown and needs to be estimated having access only to a finite sample from covariates distribution $\Ptest(X)$. This task is known as quantification \citep{gonzalez-review-quantification, forman}. 
 
Although quantification found applications in adjusting the classifier predictions, it is an important problem on its own. For example, imagine an inaccurate but cheap COVID-19 test,
which can be taken by a significant fraction of the population on a weekly basis. While this test may not be sufficient to determine whether a particular person has COVID-19, the estimate of the true number of positive cases could be used by epidemiologists to monitor the reproduction number and by the health authorities to inform public policy\footnote{Note that outbreaks induce correlations between observed data, violating the usual assumption that the data are exchangeable. We discuss contraindications in Subsection \ref{subsection:societal-impact}.}.

We advocate treating the quantification problem using Bayesian modelling, which allows estimating the uncertainty attached to the $\Ptest(Y)$ estimate.
This uncertainty can be used directly if the distribution on the whole population is of interest, or it can be used to calibrate a probabilistic classifier to yield a more informed estimate for the label of a particular observation.

A Bayesian approach was already proposed by \citet[\textsection 6]{Storkey2009}.
However, that proposal relies on a generative model $P(X\mid Y)$, which is often intractable in high-dimensional settings. 
Hence, quantification is usually approached either via the expectation maximization (EM) algorithm \citep{peters-coberly, Saerens-2001-adjustingtheoutputs} or a family of closely-related algorithms known as invariant ratio estimators \citep{Vaz-Izbicki-Stern}, black-box shift estimators \citep{Lipton2018}, or adjusted classify and count \citep{forman}, which replace the generative model $P(X\mid Y)$ with a (potentially biased) classifier.
\citet{Tasche2017}, \citet{Lipton2018}, and \citet{Vaz-Izbicki-Stern} proved that these algorithms are asymptotically consistent (they rediscover $\Ptest(Y)$ in the limit of infinite data) under weak assumptions and derived asymptotic bounds on the related error.

Our contributions are:
\begin{enumerate}
    \item For the first time, we show how to interpret this family of algorithms as an approximation of the (usually intractable) inference in the Bayesian setting.
    \item We present a tractable approach well suited for low data situations. Established alternatives provide asymptotic estimates on error bounds, but may be far off for small samples (to the point that some of the estimates for $\Ptest(Y)$ may be negative).
    Our approach explicitly quantifies the uncertainty and does not suffer from the negative values problem.
    Moreover, it is possible to incorporate expert's knowledge via the choice of the prior distribution.
    \item We prove that the \emph{maximum a posteriori} inference in our model is asymptotically consistent under weak assumptions.
\end{enumerate}

\section{Bayesian Quantification}

Consider an object with label $Y$, represented by a random variable (r.v.) valued in $\mathcal Y = \{1, 2, \dotsc, L\}$, and features $X$ (r.v.~with values in some set $\mathcal X$).
We consider an anti-causal problem in which there exists a (non-deterministic) mechanism $P_{\theta^*}(X\mid Y)$, responsible for generating the features from the label.

\newcommand{\Plab}{P_\text{lab}}
\newcommand{\Punl}{P_\text{unl}}
\newcommand{\Mtrue}{\mathcal M_\text{true}}
\newcommand{\Mapprox}{\mathcal M_\text{approx}}
\newcommand{\Msmall}{\mathcal M_\text{small}}

We consider two populations (``labeled'' and ``unlabeled'') sharing the same causal mechanism $\Plab(X\mid Y) = \Punl(X\mid Y) = P_{\theta^*}(X\mid Y)$, but which can differ in the prevalence of class labels, i.e., $\Plab(Y) \neq \Punl(Y)$.
We usually have only a finite sample from each of these, so we will use a probabilistic graphical model $\Mtrue$ (see Fig.~\ref{fig:bayes-networks}).

\begin{figure}[t]
    \begin{minipage}[t]{0.4\linewidth}
        \centering
        \begin{tikzpicture}[scale=0.85,every node/.style={transform shape}]
        \node[latent] (theta) {$\theta$};

        \node[text width=1cm] at (-1.65, -1.8) {\Large $\Mtrue$};

        \node[obs, above=of theta, yshift=-0.35cm] (yi) {$Y_i$};
        \node[obs, right=of yi] (xi) {$X_i$};
        \node[latent, left=of yi] (pi) {$\pi$};

        \node[latent, below=of theta, yshift=0.75cm] (yi1) {$Y_j'$};
        \node[latent, left=of yi1] (pi1) {$\pi'$};
        \node[obs, right=of yi1] (xi1) {$X_j'$};

        \edge {pi} {yi};
        \edge {yi} {xi};
 
        \edge {pi1} {yi1};
        \edge {yi1} {xi1};
        \edge {theta} {xi,xi1};
        \plate {train} {(yi)(xi)} {$N$};
        \plate {test} {(yi1)(xi1)} {$N'$};
        \end{tikzpicture}
    \end{minipage}%
    \begin{minipage}[t]{0.6\linewidth}
    \centering
        \begin{tikzpicture}[scale=0.85,every node/.style={transform shape}]
        \node[latent] (phi) {$\varphi$};

        \node[text width=1cm] at (-1.7, -1.8) {\Large $\Mapprox$};

        \node[obs, above=of phi, yshift=-0.3cm] (yi) {$Y_i$};
        \node[obs, right=of yi] (ci) {$C_i$};
        \node[latent, left=of yi] (pi) {$\pi$};

        \node[latent, below=of phi, yshift=0.7cm] (yi1) {$Y_j'$};
        \node[latent, left=of yi1] (pi1) {$\pi'$};
        \node[obs, right=of yi1] (ci1) {$C_j'$};

        \edge {pi} {yi};
        \edge {yi} {ci};
 
        \edge {pi1} {yi1};
        \edge {yi1} {ci1};
 
        \edge {phi} {ci,ci1};

        \plate {train} {(yi)(ci)} {$N$};
        \plate {test} {(yi1)(ci1)} {$N'$};
        \end{tikzpicture}
    \end{minipage}%
    \caption{Left: High-dimensional model $\Mtrue$. Right: tractable approximation $\Mapprox$. Filled nodes represent observed r.v., top row represents the labeled data set and the bottom row represents the unlabeled data set.}
    \label{fig:bayes-networks}
\end{figure}
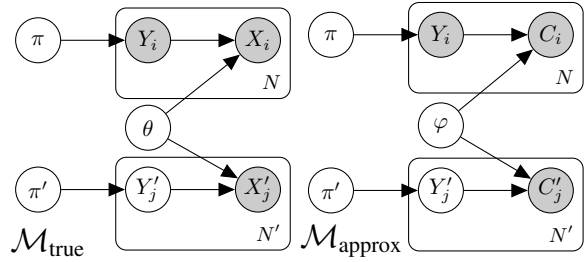

Parameters of the true generative mechanism $\theta^*$ are not known, so  we model them with a latent r.v.~$\theta$.
Then, prevalence of $Y$ in both populations is modelled by r.v.~$\pi$ for $\Plab(Y)$ and $\pi'$ for $\Punl(Y)$, valued in the probability simplex\footnote{Note that we use the open simplex. In particular, we assume that each label $l\in \mathcal Y$ has a non-zero probability of occurring under both $\Plab$ and $\Punl$.}
\(
    \Delta^{L-1} = \left\{ y \in {(0, 1)}^L \colon y_1 + \cdots + y_L = 1 \right\}.
\)

The labels are sampled from the categorical distributions:
\begin{align*}
    Y_i \mid \pi &\sim \mathrm{Categorical}(L, \pi), & i= 1,\dotsc, N,\\
    Y_j' \mid \pi' &\sim \mathrm{Categorical}(L, \pi'), & j = 1\dotsc, N',
\end{align*}
and the mechanism generating features $X$ from label $Y$ is assumed to be unchanged, i.e., 
\begin{align*}
    X_i\mid y_i, \theta  &\sim P_{\theta}(X\mid Y=y_i), & i = 1,\dotsc, N,\\
    X_j'\mid y_j', \theta  &\sim P_{\theta}(X\mid Y=y'_j), & j = 1,\dotsc, N'.
\end{align*}
where $P_{\theta}(X\mid Y)$ is the generative mechanism, which can be arbitrarily complex distribution.
Note that this is often called the prior probability shift assumption, i.e., for every $l\in \mathcal Y$, $\Plab(X\mid Y = l) = \Punl(X\mid Y=l)$ \citep{Lipton2018, Vaz-Izbicki-Stern, Tasche2017}.

In the Bayesian setting, solving the quantification problem amounts\footnote{For probabilistic classifier recalibration, it seems promising to reparametrize this model to infer the entry-wise quotient $Q=\pi'/\pi$, rather than propagating the uncertainty from $P(\pi, \pi')$. We will, however, not pursue this direction in this work.}
to inferring the posterior \(P(\pi, \pi' \mid \{X_i, Y_i\}, \{X_j'\}).\)
In many cases\footnote{  
The posterior inference on $\pi$ can be done analytically, if one assumes that the prior $P(\pi, \pi', \theta)$ factorizes and $P(\pi)$ is modelled with a Dirichlet prior. 
} the training data set may be sufficiently large to replace $\pi$ by a (maximum likelihood) point estimate and it would suffice to infer $P(\pi' \mid \{X_i, Y_i\}, \{X_j'\})$. 
Unfortunately, inference of $\pi'$ in this model is intractable whenever $X$ is high-dimensional and the generative mechanism $P_\theta(X\mid Y)$ is complex, as one needs to marginalize $\theta$ out.

\subsection{Approximating the Model}
\label{subsection:approximating-the-model}

One possible solution to circumvent the problem is to replace high-dimensional features with a simpler representation, which may be easier to model.
Consider an auxiliary space $\mathcal C$ and a mapping\footnote{Although we use notation associated with set-theoretic functions, our results hold mutatis mutandi even if the mapping is not deterministic, i.e., $f$ can be some sampling procedure from $P(C\mid X)$.} $f\colon \mathcal X\to \mathcal C$ defining new r.v.~$C_i = f(X_i)$ and $C_j' = f(X_j')$. 






 
  



If the space $\mathcal C$ is low-dimensional and $f$ is informative enough, it may be possible to model $P_\varphi(C\mid Y)$ (instead of $P_\theta(X\mid Y)$), while retaining enough information about $Y$.
In other words, one may try to do inference using an approximation $\Mapprox$ to the original model (see Fig.~\ref{fig:bayes-networks}).
For inference, this corresponds to replacing the intractable posterior $P(\pi, \pi' \mid \{X_i, Y_i\}, \{X_j'\})$ with a more tractable $P(\pi, \pi' \mid \{C_i, Y_i\}, \{C_j'\})$, supposed to be a realistic approximation to the original problem\footnote{Substituting $X_i$ for some summary statistic $C_i=f(X_i)$ loses some information and makes the learning problem harder.
However, it is effectively done in every statistical problem, in the form of feature selection or observing selected data modalities. 
}.

In terminology of \citet{Lipton2018}, this is written as $\Plab(C\mid Y=l) = \Punl(C\mid Y=l)$ and is called weak prior probability shift assumption. 
It is slightly more general than the original prior probability shift assumption, as invariance of $P(X\mid Y)$ implies invariance of $P(C\mid Y)$.
On the other hand, even if $P(X\mid Y)$ is not invariant (e.g., the image background changes), it may still be possible to learn invariant representations \citep{invariant-risk-minimization}.

By changing $\mathcal C$ and $f$ one can control the trade-off between the tractability and approximation quality: using $\mathcal C = \mathcal X$ and $f(x) = x$ gives the original problem, which is intractable but there is no approximation error; on the other hand, the trivial approximation $\mathcal C = \{1\}$ and $f(x) = 1$ forgets any available information and results in the posterior being the same as the prior $P(\pi' \mid \{C_i, Y_i\}, \{C_j'\}) = P(\pi')$ even in the limit of infinite data. 
We focus on one particular model, related to the algorithms known as quantification with black-box estimators and adjusted classify and count.

\subsection{The Discrete Model}
\label{subsection:the-discrete-model}

Consider $\mathcal C = \{1, 2, \dotsc, K\}$ and any given $f\colon \mathcal X\to \mathcal C$. The mechanism $P_{\varphi}(C\mid Y)$ is represented by a matrix
\(
    \varphi_{lk} :=  P(C=k \mid Y=l).
\)
Using the notation \( \varphi_{l:} := (\varphi_{lk})_{k\in \mathcal C}\) we can write the approximate model $\Mapprox$ as:
\begin{align}
    \pi, \pi', \varphi &\sim \text{Prior knowledge}\\
    Y_i \mid \pi &\sim \text{Categorical}(L, \pi), &i= 1, \dotsc, N \label{eq:generate_y_valid} \\
    C_i\mid y_i, \varphi &\sim \text{Categorical}(K, \varphi_{y_i:}), &i = 1, \dotsc, N  \label{eq:generate_c_valid} \\
    Y_j' \mid \pi' &\sim \text{Categorical}(L, \pi'), & j = 1, \dotsc, N' \label{eq:generate_y_test} \\
    C_j'\mid y_j', \varphi &\sim \text{Categorical}(K, \varphi_{y_j':}). & j = 1, \dotsc, N' \label{eq:generate_c_test}
\end{align}
It is convenient to model the prior on $\pi$, $\pi'$, and vectors $\varphi_{l:}$ using Dirichlet distributions, which conceptually resembles Latent Dirichlet Allocation \citep{Pritchard2000, Blei2003}, especially if several different test populations were used.
However, there are two important differences.
First, r.v.~$Y_i$ and $C_i$ are observed, which constrains the $\varphi$ matrix.
Secondly, the range $\mathcal C$ is discrete and small, rather than a list of integers. In Section \ref{subsection:fast-inference} we will show how to construct a scalable sufficient statistic and perform efficient inference using Hamiltonian Markov Chain Monte Carlo methods \citep{Betancourt2017-HMC}.

We should stress that the function $f$ does not need to retain any information (e.g., for $K=1$) and the (tractable) $P(\pi' \mid \{C_i, Y_i\}, \{C_j'\})$ may be very different from the (generally intractable) $P(\pi' \mid \{X_i, Y_i\}, \{X_j'\})$.

\subsection{Fast Inference}
\label{subsection:fast-inference}

In this section we construct a sufficient statistic for $\pi$, $\pi'$, and $\varphi$, whose size independent is of $N$ and $N'$.

Define a $K$-tuple $(N'_k)_{k\in \mathcal C}$ of r.v.~summarizing the unlabeled data set
\(
    N'_k(\omega) = \big| \{ j \in \{1, \dotsc, N'\} : C'_j(\omega) = k \} \big|,
\)
where $\omega$ is a random outcome. This can be constructed in $O(K)$ memory and $O(N')$ time: when we observe $x_1', \dotsc, x'_{N'}$ (a realization of $X_1', \dotsc, X'_{N'}$), we apply the mapping $f$ to obtain the realization $c'_1, \dotsc, c'_{N'}$ and count indices $j \in\{1, \dotsc, N'\}$ such that $c'_j = k$.
Then, for each $l \in \mathcal Y$ we define a $K$-tuple of r.v.~$(F_{lk})_{k\in \mathcal C}$, such that 
\(
    F_{lk}(\omega) = \big|\{ i \in \{1, \dotsc, N\} : Y_i(\omega) = l \text{ and } C_i(\omega) = k  \}\big|.
\)
When the labeled data set $\{(x_1, y_1), \dotsc, (x_N, y_N)\}$ is observed, we apply the function $f$ to obtain $c_i = f(x_i)$ and count different $i$ such that $(c_i, y_i) = (k, l)$.
This requires $O(LK)$ memory and $O(N)$ time.
Finally, we define an $L$-tuple of r.v.~$(N_l)_{l\in \mathcal Y}$ by
\(
    N_l = F_{l1} + \dots + F_{lK}.
\)

In Appendix \AppendixFastInference{} we prove that the likelihood
\(
    P( \{Y_i, C_i\}, \{Y'_j\} \mid \pi, \pi', \varphi)
\)
is proportional\footnote{With a proportionality constant that is a positive function of
$(N_l)_{l\in Y}, (N'_k)_{k\in C}, (F_{kl})_{k\in \mathcal C, l\in \mathcal Y}$.} to the likelihood
\(
    P\big((N_l)_{l\in Y}, (N'_k)_{k\in C}, (F_{kl})_{k\in \mathcal C, l\in \mathcal Y} \mid \pi, \pi', \varphi \big),
\)
in the smaller model $\Msmall$:
\begin{align}
    (N_l)_{l\in \mathcal Y} \mid \pi  &\sim \mathrm{Multinomial}(N, \pi),\\
    (F_{lk})_{k\in \mathcal C} \mid n_l, \varphi  &\sim \mathrm{Multinomial}(n_l, \varphi_{l:}),~l \in \mathcal Y,\\
    (N'_k)_{k\in \mathcal C} \mid \pi', \varphi  &\sim \mathrm{Multinomial}(N',  \varphi^T \pi').
\end{align}
Hence, by the factorization theorem \citep{Halmos-Savage-1949}, we constructed a sufficient statistic for the inference of $\pi$, $\pi'$, $\varphi$, which size is independent on $N$ and $N'$.
In turn, we can use the likelihood\footnote{%
As its gradient is easily computable, we can use any of the efficient Hamiltonian Markov Chain Monte Carlo algorithms \citep{Betancourt2017-HMC,Hoffman-NUTS-sampler}.
}
of $\Msmall$ to sample $\pi$, $\pi'$ and $\varphi$ from the posterior of $\pi'$ rather than from $\Mapprox$.


\subsection{Asymptotic Guarantees}

As we discussed in Subsection \ref{subsection:approximating-the-model}, the principled posterior $P(\pi, \pi' \mid \{X_i, Y_i\}, \{X'_j\})$ in $\Mtrue$ will in general be different from the posterior $P(\pi, \pi' \mid \{C_i, Y_i\}, \{C'_j\})$ in approximated model $\Mapprox$.
In particular, for $\mathcal C = \{1\}$, our posterior will be the same as the prior and no learning will occur even in the infinite data limit.
Therefore, it is natural to ask under which conditions the posterior will shrink around the true value of $\Punl(Y)$.

Our model-based approach, similarly to black-box shift Estimators \citep{Lipton2018}, invariant ratio estimators \citep{Vaz-Izbicki-Stern}, and adjusted classify and count \cite{forman} 
relies on the law of total probability:
\[
    \Punl(C=k) = \sum_{l=1}^L \Punl(C=k\mid Y=l) \Punl(Y=l).
\]
If the matrix $\Punl(C\mid Y)$ is of full rank $L$, then it is left-invertible and $\Punl(Y)$ can be obtained from $\Punl(C)$ and $\Punl(C\mid Y)$.
The former can be estimated by applying the classifier $f$ to unlabeled data and the latter by using the (weak) prior probability shift assumption $\Punl(C\mid Y)=\Plab(C\mid Y)$ allowing us to estimate it from the labeled data set.
The existing techniques use slightly different point estimators to estimate the required probabilities.
In particular, under conditions essentially equivalent to the full-rank requirement they are asymptotically identifiable \citep{Lipton2018,Vaz-Izbicki-Stern,Tasche2017}, recovering the probabilities $\Punl(Y)$ in large data limit.

In our approach we do not invert the matrix $\Punl(C\mid Y)$ (modelled with $\varphi^T$), as any degeneracy is simply reflected in the posterior (showing that we did not learn anything new about the prevalence of some classes).
However, if the full-rank condition holds, the \emph{maximum a posteriori} estimate asymptotically recovers the true parameters.\!\footnote{This is similar to the classical Bernstein--von Mises theorem linking Bayesian and frequentist inference in the large data limit.} 
In Appendix \AppendixAsymptoticIdentifiability{} we prove the following result:
\begin{theorem}
    Assume the model is not misspecified, the true $\pi^*$, $\pi'^*$, and all $\varphi_{l:}^*$ parameters lie inside the open simplices
    , the prior $P(\pi, \pi', \varphi)$ is continuous and strictly positive on the whole space, and the ground-truth $P(C\mid Y) = (\varphi^*)^T$ matrix is of full rank $L$.
    
    Then, for every $\delta > 0$ and $\varepsilon > 0$, there exist $N$ and $N'$ large enough that with probability at least $1-\delta$ the \emph{maximum a posteriori} estimate $\hat \pi, \hat \pi', \hat \varphi$ is in the $\varepsilon$-neighborhood of the true parameter values $\pi^*, \pi'^*, \varphi^*$.   
\end{theorem}


\section{Experimental Results}

\subsection{Categorical Model}
We evaluate the estimation of a prevalence vector $\pi'$ given only the black-box mapping $f\colon \mathcal X \to \mathcal C$.
Although Bayesian models provide uncertainty quantification such as credible intervals, we restrict reporting to the \emph{maximum a posteriori} estimate for fair comparison with other methods, which only provide point estimates.

\paragraph{Experimental Design} We fix the data set sizes $N$ and $N'$, the ground-truth prevalence vectors $\pi^*$ and $\pi'^*$. We construct the ground-truth matrix $P(C\mid Y)$ by choosing the ``quality'' parameter $q$ (corresponding to the true positive rate for each class in the case $L=K$) and distributing the prediction errors uniformly among other classes\footnote{
In case $K < L$ Eq. \ref{eq:p_c_y_row} is not a valid probability vector. For $l \in \{L+1, L+2, \dotsc, K\}$ we used $\varphi_{lk} = 1/K$.
}:
\begin{equation}\label{eq:p_c_y_row}
    \varphi^*_{l:}(q) = \left( \frac{1-q}{K-1}, \dotsc, \underbrace{q}_{l\text{th position}}, \dotsc, \frac{1-q}{K-1}\right).
\end{equation}
We parametrize $\pi'^*$ as
\(
    \pi'^*(r) =\left(r, \frac{1 - r}{L-1}, \dotsc, \frac{1 - r}{L-1} \right)
\)
and fix $\pi^* = (1/L, \dotsc, 1/L)$.
%
%
Then, we sample
\(
 \mathcal D = \{ (y_1, c_1), \dotsc, (y_N, c_N) \}
\)
(according to Eq. \ref{eq:generate_y_valid}--\ref{eq:generate_c_valid}) and the corresponding unlabeled data set
\(
 \mathcal D' = \{ c'_1, \dotsc, c'_{N'} \}
\)
(using Eq. \ref{eq:generate_y_test}--\ref{eq:generate_c_test})
100 times and for each sample we compare the $\ell_\infty$ (maximum discrepancy) error%
\footnote{See Appendix~\AppendixExperiments{} for a comparison using different metrics.} between the point estimate $\hat\pi'$ and $\pi'^*$.

Unless explicitly stated otherwise, experiments in this section use the default values from Table \ref{table:parameters-categorical}.

\begin{table}[h]
\caption{Default parameters used in the experiments.} \label{table:parameters-categorical}
\begin{center}
\begin{tabular}{ll}
$N$  & 1000 \\
$N'$ & 500 \\
$r$ & 0.7\\
$q$ & 0.85\\
$L$ & 5\\
$K$ & 5\\
\end{tabular}
\end{center}
\end{table}

We consider the following point estimators\footnote{In Appendix \AppendixQuantificationMethods{} we review existing quantification methods.} capable of consuming ``hard'' labels: the
black-box shift estimator (\citet{Lipton2018}, BBSE), the
invariant ratio estimator (\citet{Vaz-Izbicki-Stern}, IR),
a simple baseline ``classify and count'' approach (CC),
and point \emph{maximum a posteriori} estimates from the $\Mapprox$ model with Dirichlet prior with $\alpha=(1, \dotsc, 1)$ (flat prior; MAP-1) and $\alpha=(2, \dotsc, 2)$ (weakly informative prior; MAP-2).

\paragraph{Changing prevalence}
We investigate the impact of increasing the prior probability shift (the difference between $\pi$ and $\pi'$) by changing $r = \pi'_1 \in \{0.5, 0.6, 0.7, 0.8, 0.9\}$ and summarize the results in the first panel of Fig. 
\ref{fig:categorical-all-experiments}.
CC is adversely impacted by a strong data shift. The other estimators all perform similar to each other.

\begin{figure*}[t]
    \centering
    \includegraphics[width=0.9\linewidth]{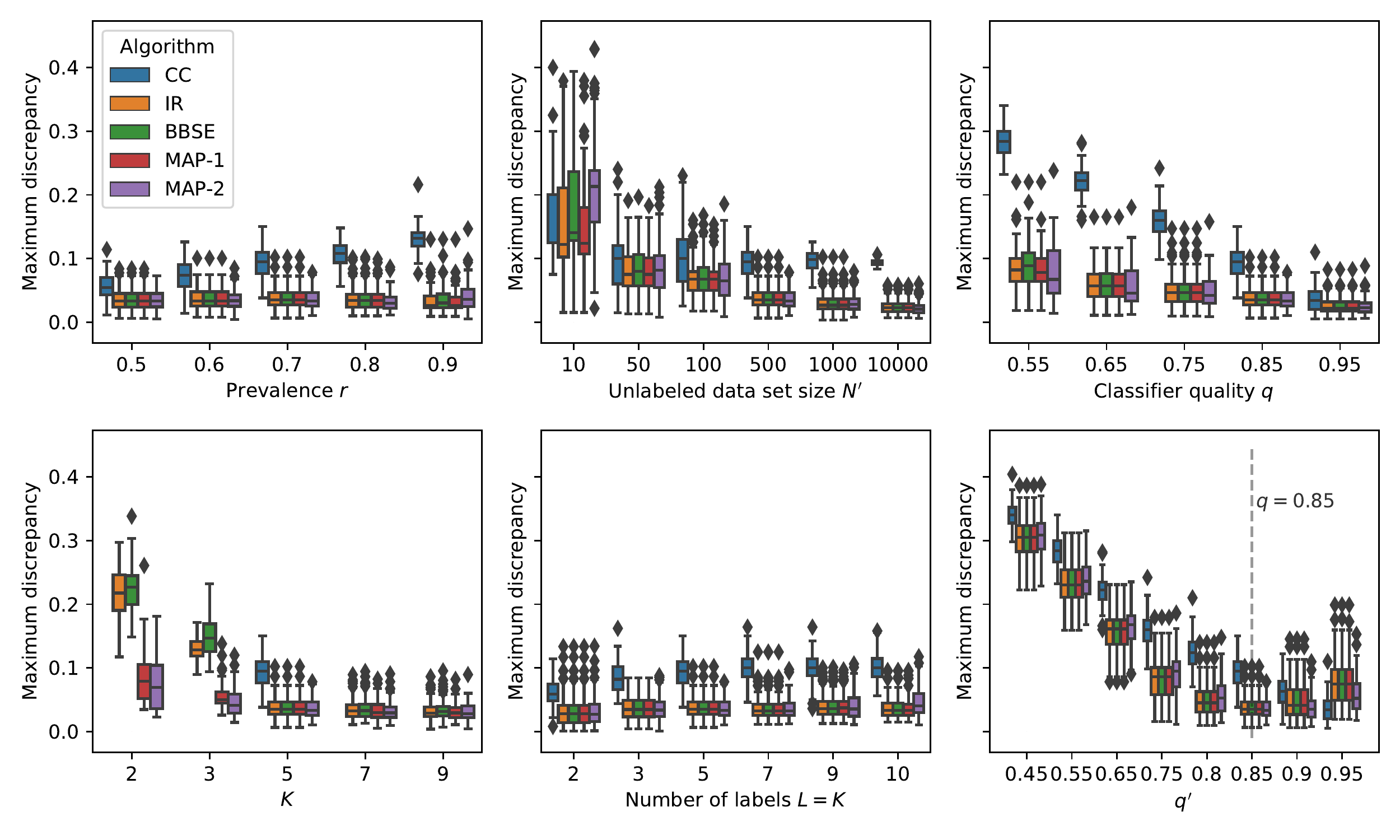}
    \caption{Quantification using simulated categorical black-box classifiers under different scenarios.}
    \label{fig:categorical-all-experiments}
\end{figure*}


\paragraph{Changing data set size}
We investigate whether the algorithms converge to the ground-truth value in the large data limit. We vary $N' \in \{ 10, 50, 100, 500, 10^3, 10^4 \}$. As shown in the second panel of Fig. \ref{fig:categorical-all-experiments}, the large data limit appears very similar (except for CC), agreeing with asymptotic identifiability guarantees for BBSE, IR and our MAP estimates. 


\paragraph{Changing classifier quality}
We investigate the impact of classifier quality (i.e. the predictive accuracy of each class) by changing it to $q \in \{0.55, 0.65, 0.75, 0.85, 0.95\}$ and show the results in the third panel of Fig. \ref{fig:categorical-all-experiments}. All considered method converge to zero error for high quality, but the convergence of CC is much slower than for the other algorithms.

\paragraph{Changing the classification granularity}
We change $K\in \{2, 3, 5, 7, 9\}$, creating a setting when a given classifier, trained on a different data distribution, is still informative about some of the classes, but provides different information. In particular, the CC estimator cannot be used for $K\neq L$. Although the original formulation of BBSE and IR assumes $K=L$, we proceed with the left inverse.
Our choice of $\varphi^*$ given above guarantees that the classifier for $K > L = 5$ will contain at least as much information as a classifier with a smaller number of classes. Conversely for $K < L$, the information about some of the classes will be insufficient even in the large data regime --- it is not possible for the matrix $P(C\mid Y)$ to have rank $L$, and asymptotic consistency does not generally hold.

The results are shown in the first panel of the second row in Fig. \ref{fig:categorical-all-experiments}. While all methods considered (apart from CC) suffer little error for $K \le L$, we note that our model-based approach can still learn something about the classes for which the classifier is informative enough, while the techniques based on matrix inversion are less effective.
Additionally, we should stress that the Bayesian approach gives the whole posterior distribution on $\pi'$ (which will not shrink), although in the plot we only compare the MAP estimates.


\paragraph{Changing the number of classes}
Finally, we jointly change $L=K \in \{2, 3, 5, 7, 9, 11, 20\}$. 
We plot the results in the fifth panel of Fig. \ref{fig:categorical-all-experiments}.
Again, classify and count obtains markedly worse results, with smaller differences between the other methods.

\paragraph{Model misspecification}
Finally, we study robustness of the considered approaches in the $\mathcal M$-open setting, i.e., when the assumption of model $\Mapprox$ (and in particular $\Mtrue$) does not hold --- the unlabeled samples are sampled according to a different $P(C\mid Y)$ distribution.
Although in this case asymptotic identifiability guarantees do not hold, we believe this to be an important case which may occur in practice (when additional distributional shifts are present).
We introduce a second matrix $\varphi'^*(q')$ and sample predictions $C'$ accordingly.
For $q' = q = 0.85$ we have $\varphi^*=\varphi'^*$, so that only prior probability shift is present.
We see that the performance of BBSE, IR and MAP estimates deteriorates for large discrepancies between $q$ and $q'$.
However, for $|q-q'| \le 0.05$, the median error of BBSE, IR and MAP is still arguably tame (although the estimator variance increases), so we hope that these methods can be employed even if the prior probability shift assumption is only approximately correct.
Note that in the case when $q' > q$, (i.e., the classifier has better predictive accuracy on the unlabeled data set than on the labeled data set, which we think rarely occurs in practice), CC outperforms other methods.

\subsection{Nearly Non-Identifiable Model}

The above experiments compare the point estimates. However, Bayesian methods shine especially at uncertainty quantification. In this section we consider a case with $L=K=3$ and
\[
    \varphi^* = (\varphi^*_{lk}) = \begin{pmatrix}
        0.96 & 0.02 & 0.02\\
        0.02 & 0.50 & 0.48\\
        0.02 & 0.48 & 0.50
    \end{pmatrix}.
\]
As the matrix is full rank, asymptotic identifiability results hold.
However, in practice classes 2 and 3 are hard to distinguish basing on the outputs of the classifier. 
We used $\pi^*=(1/3, 1/3, 1/3)$ and $\pi'^* = (0.6, 0.3, 0.1)$ and sampled data sets $\mathcal D$ and $\mathcal D'$ from the model with $N=N'=500$.

Although the point estimates can be of different quality (with BBSE returning negative probabilities), the uncertainty in our model seems to be well-calibrated, shrinking around $\pi'_1$.
The marginal variance around $\pi'_2$ and $\pi'_3$ is larger, but the joint posterior plot reveals that the model is nearly non-identifiable and only the sum of $\pi'_2 + \pi'_3 = 1-\pi'_1$ can be well-recovered in this case.

\subsection{Prevalence Estimation in Practice}

In this section we apply quantification methods to two biomedical data sets.

The classical example data set Diagnostic Wisconsin Breast Cancer Database \citep{UCI-datasets} consists of 212 malignant ($Y=1$) and 357 benign ($Y=2$) samples. We train a simple black-box classifier (random forest, \citet{scikit-learn}) for the purpose of testing the methods described in the previous section. This is a ``soft'' classifier $f\colon \mathcal X \to \Delta^{L-1}$, so we can apply the EM algorithm of \citet{Saerens-2001-adjustingtheoutputs, peters-coberly} and the soft version of the Invariant Ratio Estimator \cite{Vaz-Izbicki-Stern}, generalizing the approach of \cite{Bella2010}. For other methods, we use discretized versions.\!\footnote{See Appendix~\AppendixDiscretization{} for an overview of discretization approaches.}

Using $\pi^* = (0.5, 0.5)$ and $\pi'^* = (0.3, 0.7)$, we split the original data set into three disjoint data sets: training data (size 200 with exact proportions $\pi$) used to train a classifier $f\colon \mathcal X \to \Delta^1$; labeled data (size $N=100$ with exact proportions $\pi$) to which we apply the classifier $f$ to get the predictions and data set $\mathcal D = \{(y_1, c_1), \dotsc, (y_N, c_N)\}$; and unlabeled data (size $N'=150$ with exact proportions $\pi'$).
We only use the predictions of $f$ on the covariates to obtain $\mathcal D' = \{c'_1, \dotsc, c'_{N'}\}$.

Fig. \ref{fig:cancer_experiment} displays both the posterior learned by our model (with a flat prior used for all latent r.v.) as well as the other methods' point estimates. Arguably all are somewhat reasonable. But there is also considerable variation between the point estimates: in any practical application (say large scale drug procurement planning), it makes a huge difference whether one believes EM's predicted prevalence of 28\% or CC's predicted prevalence of 34\%. In our opinion, this stresses the importance of an approach that allows one to \emph{quantify the uncertainty} --- the procurement planner should believe neither point estimate, but choose a credible interval with a certainty corresponding to their risk profile.

\begin{figure*}[t]
    \begin{subfigure}{0.44\textwidth}
        \centering
        \includegraphics[height=3cm]{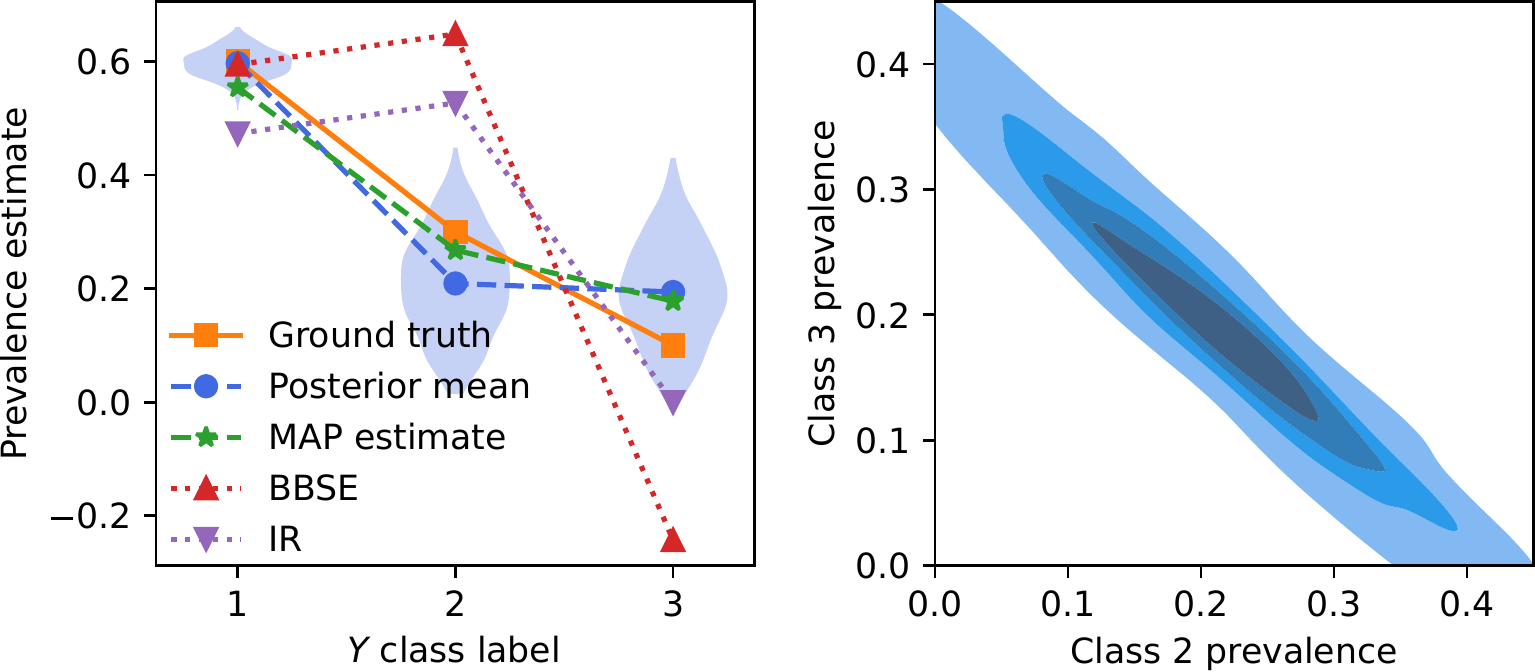}
        \caption{The posterior on $\pi'$ in nearly non-identifiable model.}
        \label{fig:nearly_non_identifiable}
    \end{subfigure}
    \begin{subfigure}{0.24\textwidth}
        \centering
        \includegraphics[height=3cm]{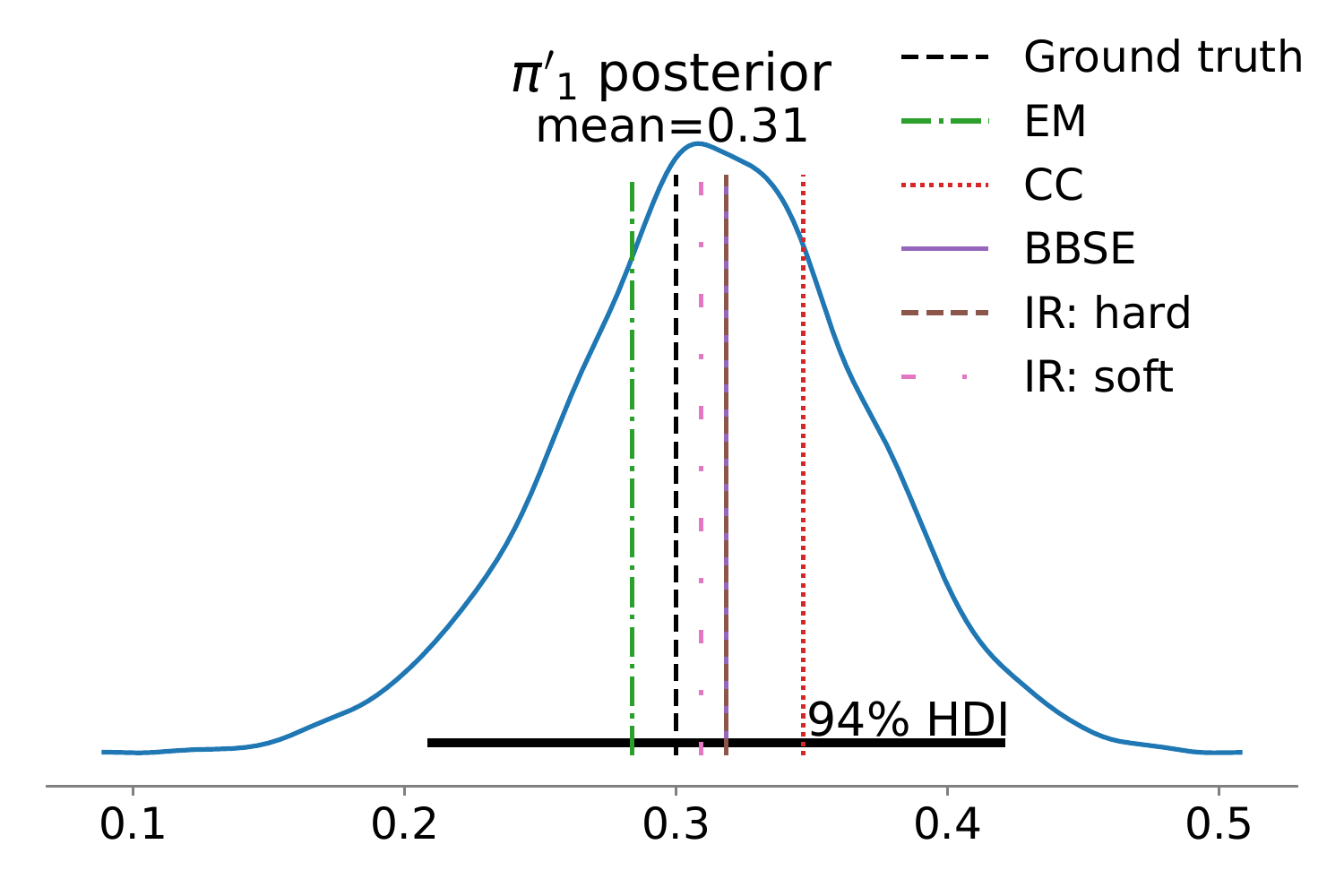}
        \caption{Breast Cancer Database.}
        \label{fig:cancer_experiment}
    \end{subfigure}
    \begin{subfigure}{0.24\textwidth}
        \centering
        \includegraphics[height=3cm]{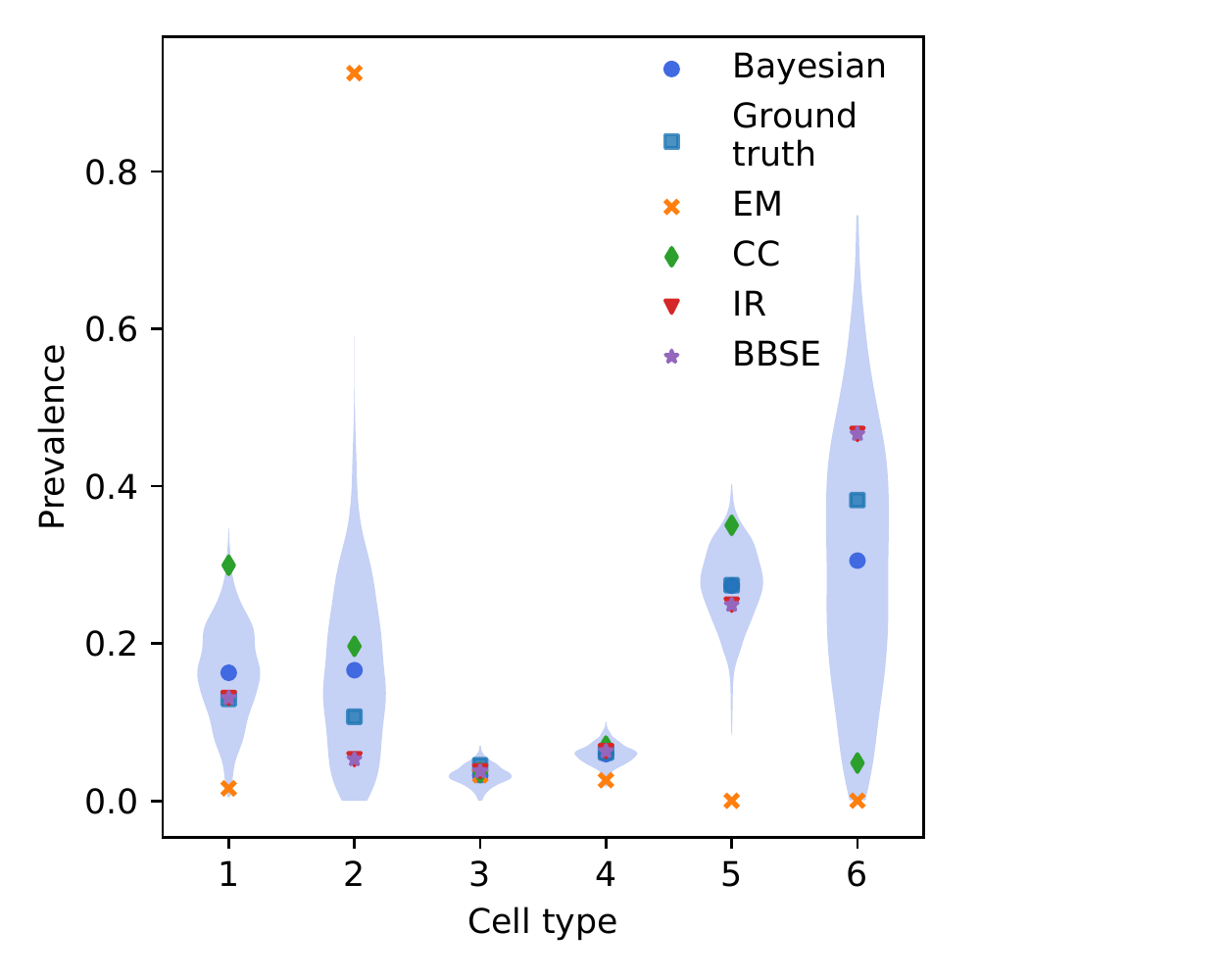}
        \caption{Single-cell data.}
        \label{fig:single-cell}
    \end{subfigure}

\caption{Bayesian posterior and point estimates in three scenarios.}
\label{fig:image2}
\end{figure*}

Then, we applied our method to single-cell RNA-seq data from Kidney Cell Atlas \citep{Stewart-kidney-cell-atlas}. 
We selected 6 types of immune cells and
selected one patient as the unlabeled data set, three patients as the validation data set and the rest of the patients as the training data set on which we trained a random forest classifier.
In Fig.~\ref{fig:single-cell} we present point estimates of different algorithms together with the posterior provided by the Bayesian approach. 
We think our method performs quite well, giving a point estimate together with uncertainty quantification.





\subsection{Uncertainty Assessment for a Gaussian Mixture}

As mentioned above, $\Mapprox$ is a tractable approximation to $\Mtrue$, but incurs a loss of information.
In this section we study the quality of this approximation in a particularly simple $\Mtrue$ model with $\mathcal X=\mathbb R$ and a mixture of two Gaussian variables. 

The generative mechanism $P_\theta(X\mid Y)$ is parameterised by means and standard deviations
\(
    \theta = (\mu_1, \mu_2, \sigma_1, \sigma_2).
\)
As $\Mtrue$ is in this case tractable, we can compare the posterior
\(
    P(\pi'\mid \{Y_i, X_i\}, \{X_j'\})
\)
in $\Mtrue$ with our suggested approximation
\(
    P(\pi'\mid \{Y_i, C_i\}, \{C_j'\})
\)
(in $\Mapprox$) for different mappings $f\colon \mathcal X \to \mathcal C$.
We partition the real line into $K$ intervals
\(
    (-\infty, a_1), [a_1, a_2), \dotsc, [a_{K-2}, a_{K-1}),  [a_{K-1}, \infty)
\)
and assign $f(x) = k$ if $x$ belongs to the $k$th interval\footnote{The ground-truth matrix $P(C\mid Y)$ corresponding to this process can be calculated analytically as
\(
P(C=k \mid Y=l) = P(X\in f^{-1}(k) \mid Y=l)
\)
and the last value is the difference of the CDF of the $l$th Gaussian distribution evaluated at the endpoints of the $k$th interval.}.

We choose $\mu_1 = 0$, $\mu_2=1$, $\sigma_1 = 0.3$, and $\sigma_2=0.4$ and sample 500 points with $Y=1$ and 500 with $Y=2$ for the labeled data set ($N=1000$).
The unlabeled data set comprises 200 instances of $Y=1$ and 800 of $Y=2$ ($N'=1000$). Then, for each $K\in \{3, 5, 7, 9\}$ we use $a_1 = -0.5$, $a_{K-1} = 1.5$ (as it captures most of the probability mass of $\Plab(X)$) and split the interval $[a_1, a_{K-1}]$ into $K-2$ evenly-spaced bins (see the left panel of Fig.~\ref{fig:gaussian_experiment}).

We fitted the Gaussian $P_\theta(X\mid Y)$ model to the data\footnote{Additional details regarding convergence and prior specification available in the Appendix TODOTODOTODO.} alongside the discrete approximations $P_\varphi(C\mid Y)$ for different $K$ using the NUTS sampler \cite{Hoffman-NUTS-sampler} and plotted the posteriors in the middle panel of Fig.~\ref{fig:gaussian_experiment} as well as the posterior means and 95\% highest density credible intervals (HDI).
Except for $K=3$, which has too wide a posterior losing too much information, all approximations and the full Gaussian model adequately capture the ground-truth ${\pi'_1}^*$ and their uncertainty estimates are in good agreement.

\begin{figure*}[t]
\begin{center}
\includegraphics[width=\linewidth]{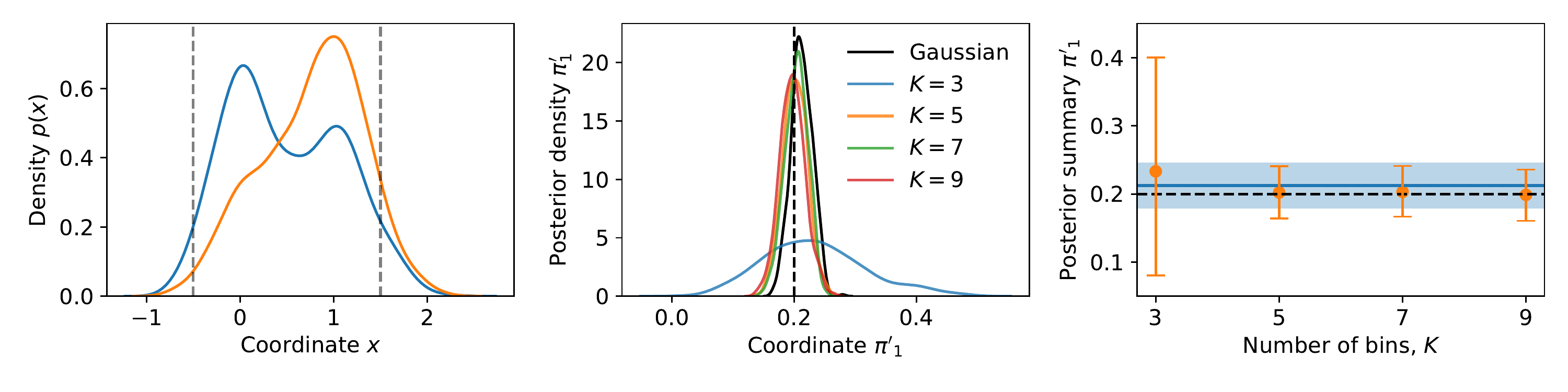}
\end{center}
\caption{
Gaussian mixture Experiment. Left: densities of $\Plab(X)$ (blue) and $\Punl(X)$ (yellow) together with lines marking $a_1$ and $a_{K-1}$. Middle: posterior density on $\pi'_1$ in different models. Dashed vertical line marks the exact $\Punl(Y=1)$. Right: dashed horizontal line marks the exact $\Punl(Y=1)$. The blue region marks the mean and the 95\% credible interval in the Gaussian mixture model. Yellow markers mark the means and the 95\% credible intervals for the discretized models.
}
\label{fig:gaussian_experiment}
\end{figure*}


\section{Connections to Prior Work}
Our method draws upon black-box shift estimators (BBSE), a quantification method proposed by \citet{Lipton2018}. This approach is closely related to invariant ratio estimators \citep{Vaz-Izbicki-Stern}, which generalize the classical adjusted classify and count algorithm (see, for example, \citet{gart-buck}, \citet[\textsection 2.3.1]{Saerens-2001-adjustingtheoutputs}, or \citet{forman}).
In particular, \cite{Lipton2018} derive theoretical guarantees including the error estimate. A variant of invariant ratio estimators, using a soft classifier, have been also proposed by \citet{Bella2010}. 

Prior to this work, \citet[\textsection 6]{Storkey2009} proposed a Bayesian approach to quantification and we believe that this approach is the most principled from Bayesian perspective if a model for $P(X\mid Y)$ is available and the inference is tractable.
As this is not the case in most machine learning problems due to limited data size, high-dimensional nature of the feature space, and complex generative models, we propose to replace $P(X\mid Y)$ with the low-dimensional approximation using an auxiliary classifier, as in the black-box shift estimators framework.
Additionally, we provide a condition under which our \emph{maximum a posteriori} estimate approximately converges to the true value in the large sample limit. We also note that using an additional classifier allows for a weaker standard prior probability shift assumption (invariance of $P(X\mid Y)$) to the invariance of $P(C\mid Y)$. 

There are several other existing quantification methods, not directly related to this work; expectation maximization \citep{peters-coberly, Saerens-2001-adjustingtheoutputs} being perhaps the most popular one. As \cite{Tasche2017} showed, expectation maximization converges to the true prevalence vector in the limit of infinite sample size. The main issue with this approach is the reliability on a calibrated classifier providing the access to the probability estimate $P(Y\mid X)$ --- as \citet{Guo2017} pointed out, modern neural networks often tend to be overconfident. A comparison between expectation maximization and black-box shift estimators can be found in \citet{garg2020}.
The CDE-Iterate algorithm of \citet{xue-weiss}, can obtain good empirical performance on selected problems \citep{karpov-etal-2016-nru}. However, as \citet[\textsection 3.4]{Tasche2017} showed, it is not asymptotically consistent.
Finally, \citet{Zhang2013} described a kernel mean matching approach, with a provable theoretical guarantee. As \citet[\textsection 6]{Lipton2018} observed, this approach is challenging to scale to large data sets.

\section{Discussion}
The presented approach generalizes point estimates provided by black-box shift estimators and invariant ratio estimators to the Bayesian inference setting.
This allows one to \emph{quantify uncertainty} and \emph{use existing knowledge} about the problem by prior specification.
Moreover, by the construction of the sufficient statistic our approach is tractable even in large-data limit 
(for either data set considered).
In all our experiments, the suggested estimator obtained at least as good performance as the existing methods, outperforming them in the $K < L$ case where the number of modelled classes differs from the ``true'' number of classes. Compared to point estimates with asymptotic guarantees, our approach ``knows what it does not know'', meaning that the posterior is meaningful even if the matrix $P(C\mid Y)$ is not \mbox{(left-)invertible}, and it is specific for the prevalence values of those classes for which the feature extractor $f$ is sufficiently informative.

More generally, we wish to stress the importance of a principal shift in perspective. Rather than training one's own classifier and then modifying that training to account for data shift, we regard $f$ as an auxiliary ``feature extraction'' method, which can be trained or tuned on an auxiliary data set in the context of an arbitrary type distribution shift. Crucial is only the access to the labeled data set which was generated according to the same process $P(C\mid Y)$. 
This is particularly useful when a hard, fully black-box classifier is given without the possibility of retraining it, which is an increasingly common theme with modern AI applications, which are often huge assets doing sophisticated processing, and also often proprietary and only available through APIs. 

However, the method we introduce is not free from challenges.
As in all Bayesian inferences, care is required regarding modelling assumptions: whether the discrete model is applicable and what prior should be used.
In particular, the prior probability shift assumption may not hold\footnote{%
  Tests for label shift are described in \citet{Lipton2018} and \cite{Vaz-Izbicki-Stern}.
}
(e.g., if the labeled and unlabeled data sets were collected under radically different conditions or the labeled and unlabeled data sets have different classes $\mathcal Y$).
Additionally, Bayesian inference often carries a model choice problem, and different choices for $K$ or the discretization method $f$ may yield different posteriors on the prevalence vector $\pi'$, especially in the low data regime.
As we remarked, if the model $P_\theta(X\mid Y)$ is tractable, we suggest to use this instead of an approximation $P_\varphi(C\mid Y)$. If it is not tractable, we suggest to use the available classifier with $K$ classes, observing the quality of $P_\varphi(C\mid Y)$ matrix, and perhaps training one's own classifier on some hold-out data set.


\subsection{Societal Impact}
\label{subsection:societal-impact}
This article discusses a Bayesian method of quantifying the prevalence of different classes in an unlabeled data set.
We note that in general the parameter posterior conditioned on the full data view $X$ can be different from the posterior conditioned on some representation $C=f(X)$ --- in cases where a reliable model $P(X\mid Y)$ is available and the inference is tractable, we suggest to use this instead of our discretized method.
Secondly, the model need not apply --- perhaps label shift is not the only distribution shift occurring in the problem or the data may not be exchangeable. In epidemiology, for example, outbreaks induce correlations between the healthiness of different people that can easily extend to sampling.
Finally, even if all the assumptions hold, recalibrating a probabilistic classifier with quantification may have undesirable consequences regarding fairness.

\subsubsection*{Code Availability and Reproducibility}

The accompanying Python implementation (including the code to reproduce the experiments) is available at
\url{https://github.com/pawel-czyz/labelshift}.


\begin{acknowledgements} 
We would like to thank Ian Wright for valuable comments on the manuscript. This publication was supported by GitHub, Inc. and ETH AI Center. We would like to thank both institutions.
\end{acknowledgements}

\newpage
\bibliography{references}
\end{document}


\onecolumn 
\maketitle

\appendix

\section{Derivation of the Sufficient Statistic}
\label{appendix:sufficient-statistic-derivation}

Starting from the joint probability
\[
    P(\pi, \pi', \varphi, \{Y_i, C_i\}, \{Y_j', C_j'\}) = P(\pi, \pi', \varphi) \times \prod_{i=1}^{N} P(C_i \mid \varphi, Y_i ) P(Y_i \mid \pi) \times \prod_{j=1}^{N'} P(C'_j \mid \varphi, Y'_j ) P(Y'_j \mid \pi'),
\]
we need to derive
\[
  P(\pi, \pi', \varphi \mid \{Y_i, C_i\}, \{C_j'\}) \propto  P(\{Y_i, C_i\}, \{C_j'\} \mid \pi, \pi', \varphi)  P(\pi, \pi', \varphi),
\]
The observed likelihood is given by marginalization of $Y'_j$ variables:
\begin{align*}
P(\{Y_i, C_i\}, \{C_j'\} \mid \pi, \pi', \varphi) &= \sum_{l_{N'} \in \mathcal Y } \cdots \sum_{l_1 \in \mathcal Y } \prod_{i=1}^{N} P(C_i \mid \varphi, Y_i ) P(Y_i \mid \pi) \prod_{j=1}^{N'}  P(C'_j \mid \varphi, Y'_j=l_j ) P(Y'_j = l_j \mid \pi)\\
&= \underbrace{\prod_{i=1}^{N} P(C_i \mid \varphi, Y_i ) P(Y_i \mid \pi)}_{A}  \times  \underbrace{\bigg(  \sum_{l_{N'}\in \mathcal Y } \cdots \sum_{l_1 \in \mathcal Y }  \prod_{j=1}^{N'} P(C'_j \mid \varphi, Y'_j = l_j ) P(Y'_j=l_j \mid \pi') \bigg)}_{B}.
\end{align*}

Each of these terms will be calculated separately.

We want to calculate
\[
    A := \prod_{i=1}^{N} P(C_i = c_i \mid \varphi, Y_i = y_i ) P(Y_i = y_i \mid \pi) = \underbrace{  \prod_{i=1}^N P(C_i=c_i \mid \varphi, Y_i = y_i) }_{A_1} \times  \underbrace{  \prod_{i=1}^N P(Y_i = y_i \mid \pi ) }_{A_2}.
\]

The term $A_2$ is simple to calculate: as $P(Y_i=y_i\mid \pi) = \pi_{y_i}$, we have
\[
   A_2 = \prod_{i=1}^N \pi_{y_i} = \prod_{l=1}^L (\pi_l)^{n_l},
\]
where $n_l$ is the number of $i\in \{1, \dotsc, N\}$, such that $y_i = l$. 
In particular, up to a factor $N! / n_1! \dots n_L!$, this is the PMF of the multinomial distribution parametrised by $\pi$ evaluated at $(n_1, \dotsc, n_L)$.

To calculate $A_1$ we need to observe that $P(C_i = k \mid \varphi, Y_i = l ) = \varphi_{lk}$. Hence,
\[
    A_1 = \prod_{i=1}^N P(C_i=c_i \mid \varphi, Y_i = y_i) = \prod_{l=1}^L \prod_{k=1}^K (\varphi_{lk})^{f_{lk}},
\]
where $f_{lk}$ is the number of $i \in \{1, \dotsc, N\}$, such that $y_i = l$ and $c_i = k$. Observe that $n_l = f_{l1} + \cdots + f_{lK}$.

In particular, up to the factor 
\[
    \prod_{l=1}^L \frac{ n_l!}{ f_{l1}! \dots f_{lK}! }
\]
this corresponds to the product of PMFs of $L$ multinomial distributions parametrised by probabilities $\varphi_{l:}$ evaluated at $f_{l:}$. 

Recall that
\[
    B :=  \sum_{l_{N'} \in \mathcal Y } \cdots \sum_{l_1 \in \mathcal Y }  \prod_{j=1}^{N'} P(C'_j = c'_j \mid \varphi, Y'_j = l_j) P(Y'_j=l_j \mid \pi').
\]

We can use the sum-product identity
\[
    \sum_{l_{N'} \in \mathcal Y } \cdots \sum_{l_1 \in \mathcal Y } \prod_{j=1}^{N'} f_j( l_j ) = \prod_{j=1}^{N'} \sum_{l \in \mathcal Y} f_j(l) 
\]
to reduce:
\[
    B = \prod_{j=1}^{N'} \sum_{l \in \mathcal Y}  P(C'_j = c'_j \mid \varphi, Y'_j = l) P(Y'_j=l \mid \pi').
\]
Because both $C'_j$ and $Y'_j$ are parametrised with categorical distributions, we have
\[
    P(C'_j = k\mid \varphi, Y'_j = l) = \varphi_{lk}
\]
and
\[
    P(Y'_j = l \mid \pi') = \pi'_l,
\]
so
\[
    \sum_{l\in \mathcal Y} P(C'_j = k \mid \varphi, Y'_j = l) P(Y'_j=l \mid \pi') = (\varphi^T \pi')_k.
\]

Hence,
\[
    B = \prod_{j=1}^{N'} (\varphi^T \pi')_{c'_j} = \prod_{k=1}^K \big( (\varphi^T \pi')_k \big)^{n'_k}, 
\]
where $n'_k$ is the number of $j \in \{1, \dotsc, N'\} $ such that $c'_j = k$.

In particular, up to a factor of $N'! / n'_1! \cdots n'_K!$, this is the PMF of the multinomial distribution parametrized by probabilities $\varphi^T\pi'$ evaluated at $(n'_1, \dotsc, n'_K)$.

\section{Proof of Asymptotic Identifiability}
\label{appendix:proof_theorem_asymptotic}

We first need to establish two simple lemmas regarding approximate left inverses:

\newcommand{\Reals}{\mathbb{R}}

\begin{lemma}
    Choose any norms on the space of linear maps $\Reals^L \to \Reals^K$ and $\Reals^K \to \Reals^L$. Suppose $K\ge L$ and that $A_0\colon \mathbb \Reals^L \to \Reals^K$ is of full rank $L$.
    Then, for every $\varepsilon > 0$ there exists $\delta > 0$ such that if
    $A\colon \Reals^L\to \Reals^K$ is any matrix such that
    \[
        || A - A_0 || < \delta,
    \]
    then the left inverse $A^{-1} := (A^TA)^{-1}A^T$ exists and
    \[
        || A^{-1} - A^{-1}_0 || < \varepsilon.
    \]
\end{lemma}

\begin{proof}
    First note that indeed the choice of norms does not matter, as all norms on finite-dimensional vector spaces are equivalent.
    
    Then, observe that rank is a lower semi-continuous function, so that for sufficiently small $\delta$ the map $A$ will be of rank $L$ as well.
    
    Finally, it is clear that the chosen formula for the left inverse is continuous as a function of $A$.
\end{proof}

\begin{lemma}
    If $K\ge L$ and matrix $A_0\colon \Reals^L\to \Reals^K$ is of full rank $L$, then for every $\varepsilon > 0$ there exist numbers $\delta > 0$ and $\nu > 0$ such that for every linear mapping $A\colon \Reals^L\to \Reals^K$ and vector $v\in \Reals^L$ if
    \[
        ||A - A_0|| < \delta
    \]
    and
    \[
        ||Av-A_0v_0|| < \nu,
    \]
    then
    \[
        ||v - v_0|| < \varepsilon.
    \]
\end{lemma}

\begin{proof}
    Again, the norm on either space can be chosen arbitrarily without any loss of generality. We will choose the $p$-norm for vectors and the induced matrix norms. 
    
    From the previous lemma we know that for any chosen $\beta > 0$ we can take $\delta > 0$ such that $A$ is left-invertible and
    \[
        || B - B_0 || < \beta,
    \]
    where $B=A^{-1}$ and $B_0 = A_0^{-1}$ are the left inverses in the form defined before. 
    
    Write $w = Av$ and $w_0 = A_0v_0$. We have
    \begin{align*}
        ||v-v_0|| &= || Bw - B_0w_0 || \\
        &= || (Bw - B_0 w) + (B_0w - B_0w_0) || \\
        &= || (B-B_0)w + B_0(w - w_0) || \\
        &\le || (B-B_0)w || + || B_0(w - w_0) || \\
        &\le ||B-B_0||\cdot  ||w|| + ||B_0|| \cdot ||w-w_0|| \\
        &\le \beta ||w|| + ||B_0|| \nu.
    \end{align*}
    
    We can bound each of these two terms by $\varepsilon/3$ choosing appropriate $\beta$ and $\nu$. Then, we can find $\delta$ yielding appropriate $\beta$.
\end{proof}


Now the proof will proceed in two steps:
\begin{enumerate}
    \item We show than for any prescribed probability we can find $N$ and $N'$ large enough that the maximum likelihood solution will be close to the true parameter values. 
    
    \item Then, we show that for reasonable priors the maximum a posteriori solution will almost surely assymptotically converge to the maximum likelihood solution. 
\end{enumerate}

\newcommand{\trueprevalencetrain}{{\pi^*}}
\newcommand{\trueprevalencetest}{{\pi'^*}}
\newcommand{\trueconfusion}{{\varphi^*}}

Let's assume that the data was sampled from the model with true parameters $\trueprevalencetrain$, $\trueprevalencetest$, $\trueconfusion$ and take $\delta > 0$ and $\varepsilon > 0$. 

\newcommand{\prevalencetrainestimate}{\hat\pi}
\newcommand{\prevalencetestestimate}{\hat \pi'}
\newcommand{\confusionestimate}{\hat\varphi}

For any $\nu > 0$ we can use the fact that log-likelihood is given by   
\[
    \ell(\pi, \pi', \varphi) =  \sum_{l\in \mathcal Y} N_l \log \pi_l + \sum_{k\in \mathcal C} \sum_{l\in \mathcal Y} F_{lk} \log \varphi_{lk} + \sum_{k\in \mathcal C} N'_k \log (\varphi^T\pi')_k,
\]
and  by the strong law of large numbers we can find $N$ and $N'$ large enough that with probability at least $1-\delta$ we will have
$|| \prevalencetrainestimate - \trueprevalencetrain || < \nu$ and $|| \confusionestimate - \trueconfusion || < \nu$, and  $|| \confusionestimate^T \prevalencetestestimate - \trueconfusion^T \trueprevalencetest || < \nu$,
where $\prevalencetrainestimate$, $\confusionestimate$, and $\prevalencetestestimate$ is the maximum likelihood estimate.

Basing on the previously established lemmas we conclude that we can pick $\nu$ small enough that 
$|| \prevalencetrainestimate - \trueprevalencetrain || < \varepsilon$, $|| \confusionestimate - \trueconfusion || < \varepsilon$, and  $||\prevalencetestestimate - \trueprevalencetest || < \varepsilon$.

Now note that if we assume the PDF of the prior $P(\pi, \pi', \varphi)$ to be continuous, we can take a compact neighborhood of $(\trueprevalencetrain, \trueprevalencetest, \trueconfusion)$ inside $\Delta^{L-1} \times \Delta^{L-1} \times \Delta^{K-1} \times \cdots \times \Delta^{K-1}$ with probability mass arbitrarily close to $1$. Then, the log-prior defined on this set will be bounded and the \emph{maximum a posteriori} estimate can be made arbitrarily close to the maximum likelihood estimate with any desired probability.

\section{Additional Experiments}
\label{appendix:additional-experiments}

\subsection{Discrete Categorical Model}
In Fig.~\ref{fig:categorical-l1} and Fig.~\ref{fig:categorical-l2} we present the comparison between different point estimators using different loss functions (mean absolute error and mean squared error).

\begin{figure*}
    \centering
    \includegraphics[width=\textwidth]{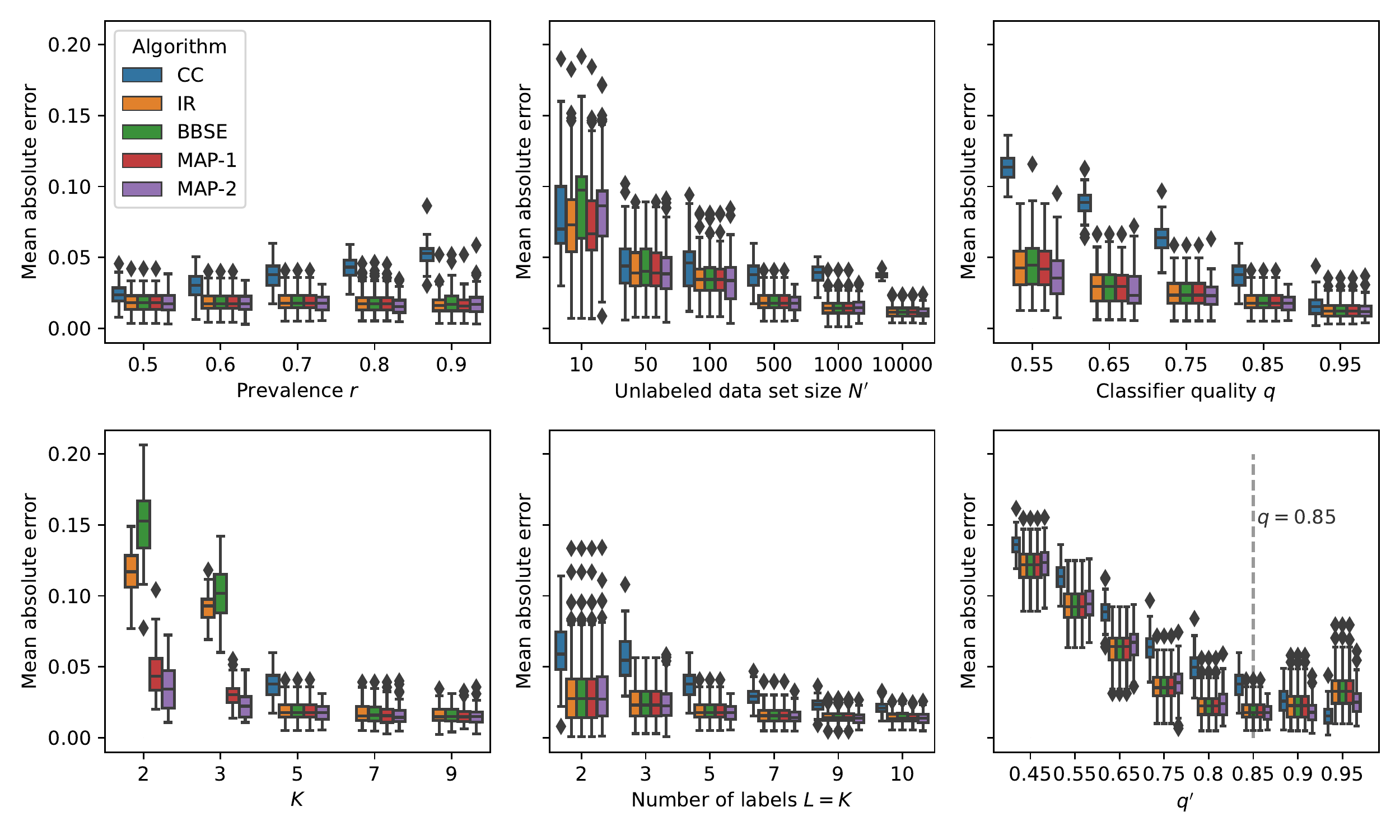}
    \caption{The results of the experiments with simulated categorical black-box classifier using mean absolute (mean $\ell_1$) error.}
    \label{fig:categorical-l1}
\end{figure*}

\begin{figure*}
    \centering
    \includegraphics[width=\textwidth]{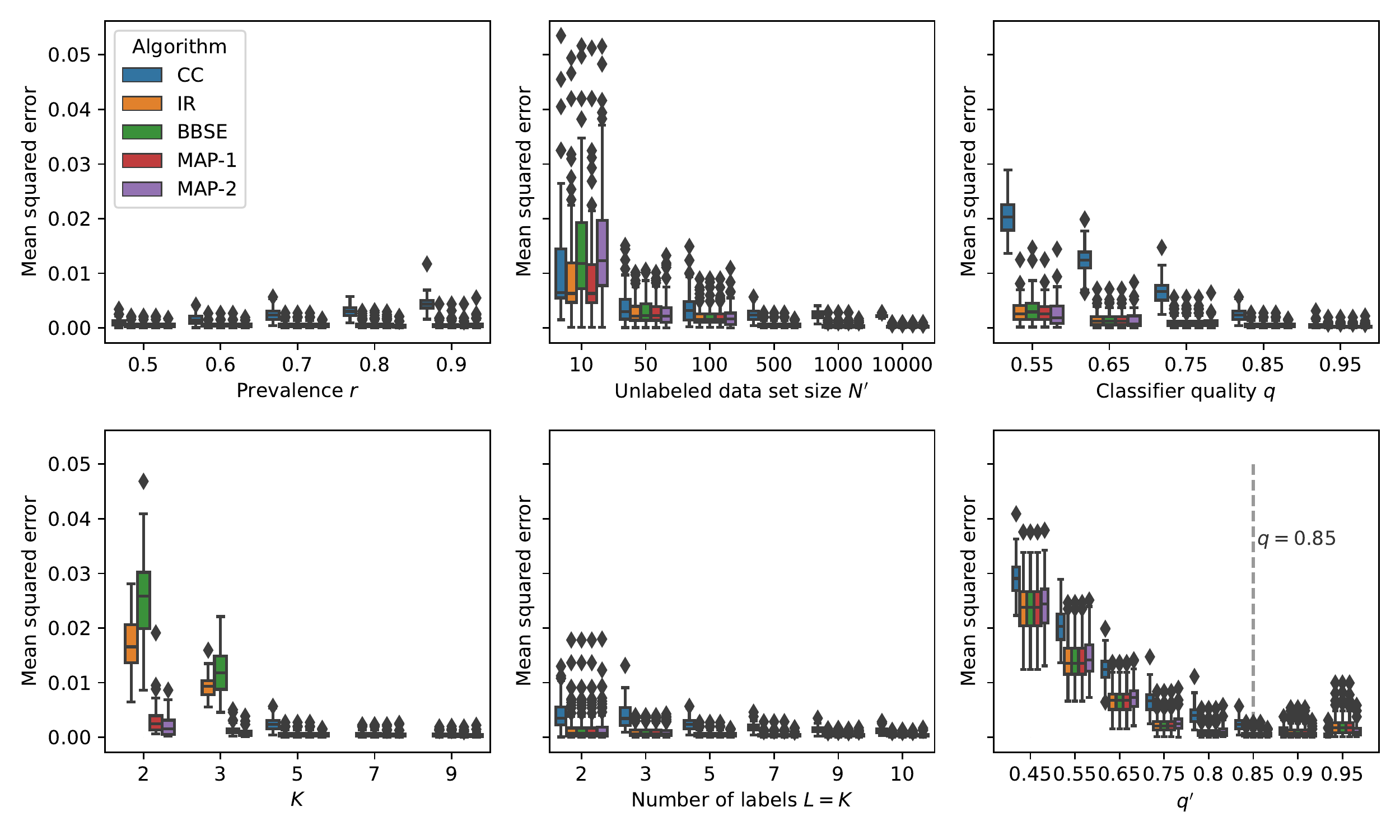}
    \caption{The results of the experiments with simulated categorical black-box classifier using mean squared (mean $\ell_2$) error.}
    \label{fig:categorical-l2}
\end{figure*}

\subsection{Gaussian Mixture Model}
For the Gaussian model we specified the priors on means using the normal distribution $\mu_1, \mu_2\sim \mathcal N(0, 1)$ and on the standard deviations using the half-normal distribution $\sigma_1, \sigma_2\sim |\mathcal N|(0.5^2)$.
We used the same prior for $\pi'$ as for the discrete models: $\pi' \sim \mathrm{Dirichlet}(\alpha)$, with $\alpha=(2, 2)$.
From each model we sampled three chains with  $10\,000$ samples using the NUTS sampler \citep{Hoffman-NUTS-sampler} and PyMC \citep{pymc} discarding the first ``warm up'' $5\,000$ samples. The Gelman--Rubin statistic and manual investigation of the trace plots did not unravel convergence problems.
To plot the posteriors we used a KDE plot on one of the thinned (by a factor of 10) chains.

\section{Discretizing Continuous-Valued Classifiers}
\label{appendix:discretizing-continuous-valued-classifier}

Consider first a \emph{soft} classifier $\tilde f\colon \mathcal X\to \Delta^{L-1}$ which outputs a \emph{confidence vector} $\tilde f(x) \in \Delta^{L-1}$.
Note that the confidence vector often does not need to represent the true conditional probability \citep{Guo2017}.

To obtain a \emph{discrete} classifier $f$, we will generalize the usual operation
\[
    f(x) = \argmax_{l = 1, \dotsc, L}~ [\tilde f(x)]_l.
\]
using partitions. Define the random variable $\tilde C = \tilde f(X)$, where $X$ is the random variable corresponding to the available covariates.
\begin{definition}
    \label{definition:partition}
    Call a family of sets $A_1, \dotsc, A_K$, where $A_k\subseteq \Delta^{L-1}$, a \textbf{partition for $\tilde C$} if:
    \begin{itemize}
        \item $P(\tilde C\in A_i \cap A_j) = 0$ for $i\neq j$ and
        \item $P(\tilde C \in A_1 \cup \dots \cup A_K) = 1$.
    \end{itemize}
    for both $\Plab$ and $\Punl$.
\end{definition}

Given a partition we can convert a soft classifier $\tilde f$ which obtains the values in the $(L-1)$-simplex into a hard classifier $f$ which obtains the values in the discrete set $\mathcal C = \{1, \dotsc, K\}$. Changing $K$ we can retain more or less information on the problem.

For a binary classifier $L=2$ the simplex $\Delta^1$ can be parametrized by the interval $(0, 1)$. It is natural to partition it into $K=2$ sets basing on a value $\alpha \in (0, 1)$, what is related to Platt scaling \citep{platt-scaling}.

For a general $L$ the partitions can be defined e.g., by applying clustering algorithms to the predictions $\tilde f (x)$. Heuristically, we expect this technique can be used to create better $P(C\mid Y)$, but we leave testing this technique to future  work. We also note that discretization should be applied with care, as it incurs information loss.


\section{Quantification Estimators}
\label{appendix:quantification-estimators}

\subsection{Classify and Count}
When $\mathcal C = \mathcal Y$ and $f\colon \mathcal X \to \mathcal Y$ is a classifier trained for a given problem with good accuracy, the simplest approach is to count its predictions and normalize by the total number of examples in the unlabeled data set.

\subsection{Adjusted Classify and Count}
Consider a case of an imperfect binary classifier, with $\mathcal Y = \mathcal C = \{+, -\}$. The true and false positive rates are defined by
\newcommand{\TPR}{\mathrm{TPR}}
\newcommand{\FPR}{\mathrm{FPR}}
\begin{align*}
    \TPR = P(C=+\mid Y=+)\\
    \FPR = P(C=+\mid Y=-)
\end{align*}

and can be estimated using the labeled data set.

If $\theta = \Punl(Y=+)$, we have
\[
    \Punl(C=+) = \TPR \cdot \theta + \FPR \cdot (1-\theta)
\]
which can be estimated by applying the classifier to the unlabeled data set and counting positive outputs.

If we assume that $\TPR \neq \FPR$, i.e., the classifier has any predictive power, we obtain
\[
    \theta = \frac{ \Punl(C=+) - \FPR }{\TPR - \FPR}.
\]

Then, $\Punl(C=+)$ is estimated by counting the predictions of the classifier on the unlabeled data set. As \citet{Tasche2017} showed, it is consistent in the limit of infinite data.

Two generalizations, extending it to the problems with more classes, are known as the invariant ratio estimator and black-box shift estimator.

\subsection{Invariant Ratio Estimator}

\citet{Vaz-Izbicki-Stern} introduce the invariant ratio estimator, generalizing the Adjusted Classify and Count approach as well as the ``soft'' version of it proposed by \citet{Bella2010}.

Consider any function $g\colon \mathcal X\to \mathbb R^{L-1}$. For example, if $f\colon \mathcal X\to \mathcal Y$ is a ``hard`` classifier, we may define $g$ as the ``one-hot encoding'' of $L-1$ labels and assign the zero vector to the last label:
\[
    g(x) = \begin{cases}
        (1, 0, \dotsc, 0) &\text{ for } f(x)=1,\\
        (0, 1, \dotsc, 0) &\text{ for } f(x)=2,\\
        \qquad\vdots \\
        (0, 0, \dotsc, 1) &\text{ for } f(x)=L-1,\\
        (0, 0, \dotsc, 0) &\text{ for } f(x)=L.
    \end{cases}
\]

Analogously, for the ``soft`` classifier $f\colon \mathcal X\to \Delta^{L-1} \subset \mathbb R^L$, $g$ may be defined as $g_k(x) = f_k(x)$ for $k \in \{1, \dotsc, L-1\}$.

Then the \emph{unrestricted} estimator $\hat \pi' \in \mathbb R^L$ is given by solving the linear system
\begin{align*}
\begin{cases}
    \hat g_1 &= \hat G_{11} \pi'_1 + \dots + \hat G_{1L} \pi'_L \\
    & \vdots\\
    \hat g_{L-1} &= \hat G_{L-1,1} \pi'_1 + \dots + \hat G_{L-1,L} \pi'_L \\
    1 &= \pi'_1 + \dots + \pi'_L
\end{cases}
\end{align*}
where 
\[
    \hat g_k = \frac{1}{N'}\sum_{j=1}^{N'} g_k(x'_j)
\]
and
\[
    \hat G_{kl} = \frac{1}{|S_l|} \sum_{ x \in S_l} g_k(x),
\]
where $S_l$ is the subset of the labeled data set $\mathcal D$ such that $y_i = l$.

Note that adjusted classify and count is a special case of the invariant ratio estimator, for a ``hard'' classifier. Similarly, the algorithm proposed by \citet{Bella2010} is a special case of invariant ratio estimator for a ``soft'' classifier.

The generalization for $K\neq L$ is immediate, with $\hat G$ becoming a $(K-1)\times L$ matrix and $\hat g$ becoming a vector of dimension $K-1$.

Finally, \citet{Vaz-Izbicki-Stern} introduce a restricted estimator $\hat \pi'_R \in \Delta^{L-1}$, which is given by a projection of $\hat \pi'_U$ onto the probability simplex. In our implementation we use the projection via sorting algorithm \citep{projections-Shalev-Shwartz-2006, projections-Blondel-2014}.

\subsection{Black-Box Shift Estimator}

Black-Box shift estimators are also based on the observation that
\[
    \Punl(C) = P(C\mid Y) \Punl(Y),
\]
where $P(C\mid Y)$ matrix can be estimated using either labeled or the unlabeled data set.
Instead of solving this matrix equation directly by finding the (left) inverse, \citet{Lipton2018} estimate the pointwise ratio $R(Y)=\Punl(Y) / \Plab(Y)$ by rewriting this equation as
\[
    \Punl(C) = \Plab(C, Y) R(Y),
\]
and estimate the joint probability matrix $\Plab(C, Y)$ using the labeled data set.
Then, the equation can be solved for $R(Y)$. By pointwise multiplication by $\Plab(Y)$ (estimated using the labeled data set) the prevalence vector $\Punl(Y)$ is found.

Note that this approach naturally generalizes to the $K\neq L$ case.

\subsection{Expectation--Maximization}

The expectation--maximization (EM) algorithm assumes access to a well-calibrated probabilistic classifier, representing $\Plab(Y\mid X)$ distribution and is based on two observations:
\begin{enumerate}
    \item  If we had access to $\Punl(Y=l\mid X=x)$, we could estimate the $\Punl(Y)$ vector:
    \begin{align*}
        \Punl(Y=l) 
        &= \mathbb{E}_{x\sim \Punl(X)} \big[  \Punl(Y=l\mid X=x) \big]\\
        &\approx \frac{1}{N'}\sum_{j=1}^{N'} \Punl(Y=l\mid X=x'_j).
    \end{align*}
    \item If we knew $\Punl(Y)$, we could recalibrate $\Plab(Y\mid X)$ to have $\Punl(Y\mid X)$:
    \[
        \Punl(Y=l \mid X=x) \propto \Plab(Y=l\mid X=x) \Punl(Y=l)/\Plab(Y=l).
    \]
\end{enumerate}

The EM algorithm starts with proposing an arbitrary probability distribution $\Punl(Y)$ and iterates between these two steps to the convergence. Note that each step of the algorithm requires $O(N')$ operations.

\bibliography{references}